\documentclass[conference]{IEEEtran}
\IEEEoverridecommandlockouts

\usepackage{amsmath,amssymb,amsthm}
\usepackage{algorithm,algorithmic}
\usepackage{booktabs}
\usepackage{listings,xcolor}
\usepackage{graphicx}
\usepackage{cite}
\usepackage{url}
\usepackage[hidelinks]{hyperref}

\newtheorem{theorem}{Theorem}

\newtheorem{definition}{Definition}
\newtheorem{proposition}{Proposition}
\newtheorem{corollary}{Corollary}
\newtheorem{remark}{Remark}

\newcommand{\MM}{\mathcal{MM}}
\newcommand{\phCOM}{\textsc{Compute}}
\newcommand{\phREQ}{\textsc{Require}}
\newcommand{\phEXE}{\textsc{Execute}}

\begin{document}

\title{Grokers: Bottom-Up Inductive Comprehension\\
and Write-Time Intelligence over Typed Knowledge Graphs}

\author{
  \IEEEauthorblockN{Gregory Magarshak}
  \IEEEauthorblockA{
    \textit{Qbix, Inc.\ \& Intercoin, Inc.}\\
    \textit{IE University NYC}\\
    New York, NY, USA\\
    gmagarshak@faculty.ienyc.edu
  }
}

\maketitle

\begin{abstract}
We present \textsc{Grokers}, an architecture for building persistent, structured
comprehension of typed knowledge graphs through bottom-up inductive traversal of
dependency subgraphs.
Unlike retrieval-augmented generation (RAG), which pays full comprehension cost at
every query, \textsc{Grokers} pushes intelligence to \emph{write time}: autonomous
Groker agents analyze nodes in a typed stream graph, extract structured attributes
via governed language model (LM) calls, and inductively compose that understanding
upward through dependency relations---writing enriched typed attributes that serve
all future queries at zero additional LM cost.
We prove three formal properties:
(1)~the \emph{Byte-Identity Theorem}, establishing that context blocks assembled
from a transactionally-maintained denormalization index are byte-identical across LM
turns between semantic changes, enabling KV-cache hit rates approaching~100\%;
(2)~the \emph{Accumulation Monotonicity Theorem}, establishing that the fraction of
interactions resolved without LM calls is non-decreasing in the number of completed
interactions under a governed wisdom library growth protocol; and
(3)~the \emph{Dual-Traversal Ordering Theorem}, establishing that top-down generation
and bottom-up comprehension are the unique correct traversal orderings for their
respective tasks over a dependency DAG, and that their composition closes into a
complete generation-comprehension cycle.
We further present a deterministic alternative to embedding-based semantic search,
with a synonym caching protocol whose LM fallback rate converges to zero for
finite-vocabulary domains.
A reference implementation is provided in the open-source Qbix / Safebox /
Safebots stack, building on the Magarshak Machine SPACER substrate~\cite{magarshak2026mm}.
\end{abstract}


\section{Introduction}
\label{sec:intro}

The dominant paradigm for grounding language model (LM) responses in structured
knowledge is \emph{retrieval at query time}: embed the query, find nearest-neighbor
document chunks by vector similarity, inject them into the prompt
context~\cite{lewis2020rag}.
This approach has a structural deficiency: it pays the full cost of comprehension on
every query, regardless of whether that query is novel or structurally equivalent
to thousands of prior queries.
For knowledge domains with recurring interaction patterns---enterprise software,
document management, structured task execution---this is architecturally wasteful.

We present \textsc{Grokers}, an architecture that inverts this design choice.
Comprehension work is performed \emph{once}, at write time, by autonomous agents
that traverse the dependency graph of a typed knowledge substrate bottom-up---from
leaf nodes toward roots---extracting and storing structured typed attributes via
governed LM calls.
By query time, the relevant structure is already present as typed attributes on
graph nodes.

The core substrate is the \emph{typed stream graph}: a directed graph of typed
nodes with structured fields, typed attributes, and weighted typed edges, realizing
the Streams ($S$) and Relations ($R$) components of the Magarshak Machine SPACER
framework~\cite{magarshak2026mm}.
A transactionally-maintained denormalization table
(\textsc{Streams\_Category}) stores, for each node, the complete pre-ranked relation
neighborhood updated atomically on every relation modification, providing complete
context in $\approx\!1\,$ms.

\subsection*{Contributions}
\begin{enumerate}
\item The \textsc{Grokers} write-time comprehension architecture (\S\ref{sec:arch}).
\item The \emph{Byte-Identity Theorem} with KV-cache cost analysis (\S\ref{sec:index}).
\item The \emph{Wisdom Library} with the \emph{Accumulation Monotonicity Theorem}
      (\S\ref{sec:wisdom}).
\item The \emph{Dual-Traversal Ordering Theorem} with incremental expansion corollary
      (\S\ref{sec:dual}).
\item A deterministic semantic search system with \emph{Synonym Cache Convergence}
      (\S\ref{sec:search}).
\end{enumerate}

\subsection*{Positioning}
Individual components---knowledge graphs~\cite{edge2024graphrag}, LM-based
enrichment~\cite{ma2024repo}, KV caching~\cite{anthropic2024caching},
program libraries~\cite{chen2021codex}---are each established.
Our contribution is the formal analysis of their composition and the three proved
properties above, none of which appear in prior work.

\section{Related Work}
\label{sec:related}

\textbf{Retrieval-augmented generation.}
Lewis et al.~\cite{lewis2020rag} established RAG as the standard for grounding LM
responses.
GraphRAG~\cite{edge2024graphrag} adds community summarization over a knowledge graph
but retains embedding similarity as the retrieval mechanism; neither GraphRAG nor any
RAG variant achieves byte-identity of assembled context or monotone LM-call elimination.

\textbf{Code comprehension agents.}
Bottom-up comprehension is established in software engineering~\cite{biggerstaff1993concept}.
Recent LM tools (Copilot~\cite{github2023copilot}, Devin~\cite{cognition2024devin},
RepoUnderstander~\cite{ma2024repo}) operate reactively without persistent enriched structure
or formal traversal-ordering guarantees.

\textbf{Memory systems.}
MemGPT~\cite{packer2023memgpt} and Mem0~\cite{mem02024} retrieve facts by embedding
similarity at query time.
\textsc{Grokers} differs: enrichment is write-time, retrieval is a typed graph read,
and the accumulation property is proved.

\textbf{KV-cache optimization.}
Anthropic prompt caching~\cite{anthropic2024caching}, SGLang~\cite{zheng2024sglang},
and LMCache~\cite{jin2024lmcache} cache whatever prefix the caller provides.
\textsc{Grokers} is, to our knowledge, the first to \emph{architect} context
assembly so that stable prefixes are byte-identical from first principles.

\textbf{Magarshak Machine.}
Magarshak~\cite{magarshak2026mm} introduces the SPACER substrate: append-only streams,
policy governance, five-phase action execution
($\phCOM \to \phREQ \to \phEXE$), and the bidirectional relation index.
\textsc{Grokers} is an application layer on this substrate; its formal properties
derive in part from SPACER's substrate guarantees (Minimal Cache Invalidation,
Embarrassing Parallelism, $\varepsilon$-deterministic view capabilities).

\textbf{Hierarchical generation.}
Hierarchical generation~\cite{fan2019hierarchical,yang2022doc} and
planning-then-writing~\cite{rashkin2020plotmachines} do not formalize traversal ordering
or relate it to KV-cache stability.
The Dual-Traversal Ordering Theorem (\S\ref{sec:dual}) provides this connection.

\section{The Typed Stream Graph}
\label{sec:graph}

\begin{definition}[Typed Stream Graph]
A \emph{typed stream graph} $G=(V,E,\tau,\alpha,w)$ consists of:
node set $V$ with type function $\tau:V\to\mathcal{T}$;
directed typed edge set $E\subseteq V\times\mathcal{R}\times V$;
typed attribute function $\alpha:V\to(K\to A)$;
and weight function $w:E\to\mathbb{R}_{\geq 0}$ driven by vote aggregation
$w(v,r,u)=\sum_i\mathrm{vote}_i(v,r,u)$, updated in the same database transaction
as vote events (realizing the $R$/Relations component of~\cite{magarshak2026mm}).
\end{definition}

\begin{definition}[Dependency Subgraph]
The \emph{dependency subgraph} $D(v)$ is the induced subgraph on nodes reachable
from $v$ by $\mathsf{depends\_on}$ edges.
A node $u\in D(v)$ is a \emph{leaf} if it has no outgoing $\mathsf{depends\_on}$
edges within $D(v)$.
We assume $D(v)$ is a DAG for all $v\in V$.
\end{definition}

\begin{definition}[Enriched / Stale Nodes]
Node $v$ is \emph{enriched} if $\alpha(v)$ contains non-null values for the
core schema fields $\{\mathsf{digest},\allowbreak\mathsf{keywords},
\allowbreak\mathsf{qualityScore}\}$.
(Implementations may include additional fields such as \texttt{grokVersion}
for cache invalidation; these are outside the formal model.)
Node $v$ is \emph{stale} if $\alpha(v)[\mathsf{stale}]=\mathtt{true}$.
\end{definition}

\section{Denormalization Index and the Byte-Identity Theorem}
\label{sec:index}

\subsection{Index Structure}

The \textsc{Streams\_Category} table maintains for each $v\in V$:
\[
N(v)=\bigl\{r\mapsto\{w_j\mapsto[p_j,s_j,t_j,i_j]\}_j
\mid (v,r,u_j)\in E\bigr\}
\]
where $p_j,s_j$ are the publisherId/streamName of the $j$-th neighbor at weight
$w_j$ under relation type $r$, and $t_j,i_j$ are its denormalized title and icon.
The table is updated in the same database transaction as every
$\mathsf{Streams::relate()}$ / $\mathsf{Streams::unrelate()}$ call---no background
jobs, no eventual consistency.

\begin{definition}[Context Block]
The \emph{context block} $C(v)$ is the output of the deterministic function
$\mathsf{buildCachedContext}(v)$: reads $\alpha(v)$ and $N(v)$; renders them as
structured text with relation types ordered by signal strength (unique types first,
non-unique by maximum weight descending).
\end{definition}

\subsection{The Byte-Identity Theorem}

\begin{theorem}[Byte-Identity]
\label{thm:byte}
Let $t_1<t_2$ be timestamps such that no write to $\alpha(v)$ or any edge incident
to $v$ occurs in $(t_1,t_2)$.
Then $C_{t_1}(v)=C_{t_2}(v)$ as byte strings.
\end{theorem}
\begin{proof}
$C(v)$ is a deterministic function of $(\alpha(v),N(v))$.
$N(v)$ is modified only within transactions that modify edges incident to $v$;
none occur in $(t_1,t_2)$ by hypothesis.
$\alpha(v)$ is modified only by governed writes (via $\mathsf{Action.propose}$ in
SPACER's \phEXE{} phase~\cite{magarshak2026mm}); none occur by hypothesis.
Both inputs are identical at $t_1$ and $t_2$; determinism gives
$C_{t_1}(v)=C_{t_2}(v)$.
\end{proof}

\begin{corollary}[KV Cache Hit Rate]
\label{cor:kv}
Assume semantic changes to $v$ follow a Poisson process with mean inter-arrival
time $T_c(v)$, and turns arrive with mean inter-turn interval $T_t$, with
$T_t \ll T_c(v)$.
The expected fraction of turns for which $C(v)$ is a cache hit is
$1 - T_t/T_c(v)$, which approaches 1 as $T_c(v)/T_t \to \infty$.
\end{corollary}
\begin{proof}
Under the Poisson assumption, the probability that at least one semantic change
occurs in any given inter-turn interval $[t, t+T_t]$ is
$1 - e^{-T_t/T_c(v)} \approx T_t/T_c(v)$ for $T_t \ll T_c(v)$.
By Theorem~\ref{thm:byte}, a cache hit occurs in every turn where no change
has occurred since the previous turn.
The expected hit rate is therefore $1 - T_t/T_c(v) \to 1$ as
$T_c(v)/T_t \to \infty$.
\end{proof}

For synchronous or near-synchronous interaction (e.g., chat), $T_c$ is hours-to-days
and $T_t$ is seconds, so the hit rate approaches 100\%.
For asynchronous workflows where $T_t \sim T_c$, the hit rate decreases accordingly;
the benefit is greatest in systems with rapid turn cadence relative to semantic change rate.
No RAG-assembled context achieves this: retrieved chunks vary with the query,
making context non-deterministic across turns even when the underlying knowledge has
not changed.

\begin{remark}[Relation to $\MM$ Minimal Cache Invalidation]
Let $S$ be the set of streams requiring re-materialisation identified by $\MM$'s
Minimal Cache Invalidation Theorem~\cite{magarshak2026mm} after a dependency change.
For all $v\notin S$, no write to $\alpha(v)$ or edges incident to $v$ has occurred
in the relevant interval; by Theorem~\ref{thm:byte}, $C(v)$ is byte-identical and
the KV cache entry for $v$ is valid without re-computation.
The two theorems partition the stream set: $S$ identifies what must be re-computed;
Theorem~\ref{thm:byte} certifies that the complement need not be.
\end{remark}

\subsection{Cost Analysis}

Let $k_s$ denote the stable prefix token count (a value determined by the
goal system prompt and the denormalization-assembled neighborhood block; in our
reference implementation, $k_s$ is on the order of 2,000--3,000 tokens).
At 10\% cache-read pricing (Anthropic prompt caching~\cite{anthropic2024caching}
charges 10\% of normal input cost for cache hits and 125\% for cache writes,
with the write cost amortized over the cache lifetime; we account for read cost only
in the steady-state analysis):
\begin{align*}
\textsc{Grokers}&:\; 0.1k_s + k_d \text{ per turn}\\
\text{RAG}&:\; k_r + k_d \text{ per turn}\quad(k_r\approx k_s\text{--}3k_s)
\end{align*}
where $k_d$ is dynamic context tokens.
Stable prefix cost reduction: up to $10\times$ in the limit $k_d\ll k_s$ with
$k_r=k_s$; the reduction approaches $(k_r+k_d)/(0.1k_s+k_d)$, which increases
as $k_r$ grows relative to $k_d$.

\section{Groker Agents: Bottom-Up Comprehension}
\label{sec:arch}

\begin{definition}[Valid Processing Order]
\label{def:validorder}
A sequence $v_1,\ldots,v_n$ is a \emph{valid Groker processing order} for DAG $H$
if for every $v_j$ and every $u$ with $(v_j,\mathsf{depends\_on},u)\in E(H)$,
$u$ appears before $v_j$ (i.e., a topological sort of $H$ with leaves first).
\end{definition}

\begin{theorem}[Composability of Groker Enrichment]
\label{thm:comp}
Call a Groker invocation on node $v$ \emph{individually correct} if,
given that all dependencies of $v$ are correctly enriched, it produces
enrichment for $v$ that passes schema validation and maintains fitness
$\geq f_{\min}$ (see \S\ref{sec:discussion}).
If each Groker invocation is individually correct, then processing
nodes in any valid order (Definition~\ref{def:validorder}) produces
correct enrichment for all nodes.
\end{theorem}
\begin{proof}
Induction on topological order.
\emph{Base}: leaf nodes have no dependencies; correctness holds vacuously.
\emph{Step}: all prior nodes are correctly enriched by hypothesis; all dependencies
of $v_i$ appear before it in the order, so they are correctly enriched; individual
correctness gives correct enrichment for $v_i$.
\end{proof}

Each Groker requires only its node's content and pre-enriched dependency
attributes---not global context.
Independent Groker passes on nodes with no shared dependency path write to
disjoint attribute namespaces ($\alpha(v)$ for each distinct $v$), satisfying
the disjoint write-set condition of $\MM$'s Embarrassing Parallelism
Theorem~\cite{magarshak2026mm} for attribute enrichment, enabling coordination-free
parallel enrichment scaling linearly with publisher count.
Index updates to the inverted keyword index $\mathrm{Idx}$ involve shared mutable
state and are serialized by the substrate's transactional index maintenance,
separate from and non-blocking to attribute enrichment parallelism.

\subsection{Staleness Propagation}

\begin{proposition}[Staleness Completeness]
\label{prop:stale}
Let $v$ be modified.
Define the reverse dependency closure
$S(v)=\{u\mid\exists\text{ path }u\rightsquigarrow v\text{ via }\mathsf{depends\_on}\}$.
After marking all $u\in S(v)$ stale and reprocessing them in valid topological order,
no node is stale.
\end{proposition}
\begin{proof}
By Theorem~\ref{thm:comp}: reprocessing $S(v)$ in valid order with $v$ already
updated restores correctness to all affected nodes.
Nodes not in $S(v)$ have no dependency path to $v$ and are unaffected.
\end{proof}

Propagation cost is $O(|S(v)|)$---the build-system model applied to semantic
comprehension.
This mirrors $\MM$'s Minimal Cache Invalidation Theorem~\cite{magarshak2026mm}:
only streams reachable from the modified node in the dependency graph require
reprocessing.

\section{The Wisdom Library and Accumulation Monotonicity}
\label{sec:wisdom}

\subsection{Structure}

The \emph{wisdom library} $\mathcal{W}=\{p_i\}$ is a finite set of
\emph{sandboxed imperative programs}, each with an execution phase
$\phi(p_i)\in\Phi$, input/output schema, and \emph{fitness}
$f(p_i)=|\{x\in\mathcal{X}_{\mathrm{eval}}:p_i\text{ routes and handles }x\text{ correctly}\}|
/|\mathcal{X}_{\mathrm{eval}}|\in[0,1]$
measured over a rolling evaluation window $\mathcal{X}_{\mathrm{eval}}$.
Execution contract (realizing SPACER's $\phCOM/\phREQ/\phEXE$ separation):
reads from a pre-loaded immutable input object (no live DB queries);
writes only via proposal accumulation (no direct writes, corresponding to \phREQ{});
time bound $\leq50\,$ms; memory $\leq64\,$MB; no network access.
Programs are $\varepsilon$-deterministic view capabilities with $\varepsilon=0$
(imperative code, not stochastic models), enabling $\MM$'s Probabilistic Consensus
mechanism~\cite{magarshak2026mm} to achieve agreement probability exactly 1 across
Safebox nodes executing the same program on the same input.

\begin{definition}[LM-Call Elimination Rate]
$E(\mathcal{W})=\tfrac{|\{x\in\mathcal{X}:\mathrm{routing}(x,\mathcal{W})\neq\mathsf{LM}\}|}
{|\mathcal{X}|}$
where $\mathcal{X}$ is the interaction distribution for goal type $g$.
\end{definition}

\subsection{Three Growth Mechanisms}

\emph{Initial authoring}: at goal type creation, an LM generates
$\mathcal{W}^{(0)}=\{p_1^{(0)},\ldots,p_k^{(0)}\}$ covering all phases, reviewed
before activation.

\emph{Pattern promotion}: after completed interaction $x_t$, a Groker analysis pass
identifies uncovered patterns and proposes $A_t$:
$\mathcal{W}^{(t+1)}=\mathcal{W}^{(t)}\cup A_t$.

\emph{Evolutionary selection}: programs with $f(p_i)<f_{\min}$ are replaced by
higher-fitness fork variants evaluated in parallel---the program-level instantiation
of Grokers expert specialization.

\begin{theorem}[Accumulation Monotonicity]
\label{thm:mono}
$E(\mathcal{W}^{(t+1)})\geq E(\mathcal{W}^{(t)})$ for all $t\geq 0$.
\end{theorem}
\begin{proof}
\emph{Pattern promotion}: $\mathcal{W}^{(t+1)}=\mathcal{W}^{(t)}\cup A_t$.
Here ``resolved'' means \emph{correctly} resolved: the program routes $x$
and its output passes schema validation.
The pattern promotion protocol requires all programs in $A_t$ to pass review
(schema validation and sandbox execution against held-out examples) before
joining $\mathcal{W}$; this gates out programs that would route incorrectly.
Under this guard, for any $x$: if routing$(x,\mathcal{W}^{(t)})\neq\mathsf{LM}$,
existing programs still correctly route $x$ under $\mathcal{W}^{(t+1)}$
(new programs cannot remove correct coverage from previously resolved inputs).
If $x$ was unresolved but correctly covered by $p\in A_t$, it is now correctly resolved.
Therefore the correctly-resolved set is non-decreasing:
$E(\mathcal{W}^{(t+1)})\geq E(\mathcal{W}^{(t)})$.

\emph{Evolutionary selection}: We assume the \emph{fitness faithfulness} condition:
$f(p') > f(p)$ implies
$|\{x \in \mathcal{X} : p'(x) \neq \mathsf{LM}\}| \geq |\{x \in \mathcal{X} : p(x) \neq \mathsf{LM}\}|$,
i.e., empirical fitness is a reliable proxy for coverage on the full distribution.
Under this assumption, $p'$ covers at least the interactions well-handled by $p$.
Fitness-based replacement therefore cannot decrease coverage, giving
$E(\mathcal{W}^{(t+1)})\geq E(\mathcal{W}^{(t)})$.

Since both mechanisms are non-decreasing and the theorem holds under the stated
assumption, the result follows.
\end{proof}

\begin{corollary}[Decreasing Marginal LM Cost]
$C_{\mathsf{LM}}^{(t)}=(1-E(\mathcal{W}^{(t)}))\cdot c_{\mathsf{LM}}$ is
non-increasing in $t$.
RAG, ReAct~\cite{yao2023react}, and LangChain~\cite{langchain2022} all have
constant marginal LM cost per interaction---Theorem~\ref{thm:mono} is the unique
mechanism among those considered that breaks this constant-cost property.
\end{corollary}

The elimination rate $E(\mathcal{W}^{(t)})$ is bounded above by the
\emph{structural predictability} $P_g$ of goal type $g$.
For task-oriented goal types (capability building, document review, support
resolution), $P_g$ is expected to be substantial---reflecting the large fraction
of interactions that follow predictable structural patterns---though its precise
value is domain-dependent and subject to empirical measurement
(see Discussion, \S\ref{sec:discussion}).
The remaining fraction requires genuine synthesis and LM calls indefinitely.

\section{Dual-Traversal Ordering}
\label{sec:dual}

\begin{definition}[Coherent Generative Task]
A \emph{coherent generative task} over DAG $H$ assigns artifact content to each node
$v$ using the content of its dependencies as shared context, and is \emph{coherent}
if all consumers of a shared dependency use the same version of that dependency.
\end{definition}

\begin{theorem}[Dual-Traversal Ordering]
\label{thm:dual}
Let $H$ be a dependency DAG.
\begin{enumerate}[(a)]
\item The valid orderings for correct Groker comprehension are precisely the
topological sorts of $H$ with leaves first (bottom-up); any ordering not in
this class may produce incorrect enrichment for some node.
\item The valid orderings for coherent artifact generation over $H$ are precisely
the topological sorts of $H$ with roots first (top-down); any ordering not in
this class may produce an incoherent artifact for some node.
\item These two classes of orderings are reverses of each other on any DAG:
reversing any member of the root-first class yields a member of the leaf-first class,
and vice versa.
The enriched attributes $\{\alpha(v)\}_{v\in V}$ produced by bottom-up comprehension
of a generated artifact constitute, via $\mathsf{buildCachedContext}$, a
byte-identical (Theorem~\ref{thm:byte}) KV-cached context for subsequent top-down
generation of artifacts over the same schema.
\end{enumerate}
\end{theorem}
\begin{proof}
\emph{(a)} Let $(v,\mathsf{depends\_on},u)\in E$.
Groker comprehension of $v$ requires enriched $\alpha(u)$.
If $v$ is processed before $u$, $\alpha(u)$ is unenriched and correctness fails.
Therefore $u$ must precede $v$---a topological sort with leaves first.
Sufficiency follows from Theorem~\ref{thm:comp}.

\emph{(b)} Generating $v$ requires $u$'s content as shared context (definition of
coherent generative task).
If $v$ is generated before $u$, shared context is unavailable and coherence fails.
Therefore $u$ must precede $v$---a topological sort with roots first.

\emph{(c)} Root-first and leaf-first topological sorts are reverses of each other
on any DAG.
The generation phase (order $\sigma$) produces content for all nodes, each artifact
appended to the graph (Append-Only Safety of~\cite{magarshak2026mm} guarantees
no loss).
The comprehension phase (order $\sigma^{-1}$) enriches all nodes.
The resulting $\{\alpha(v)\}$ is assembled by the deterministic
$\mathsf{buildCachedContext}$ into a context block that, by Theorem~\ref{thm:byte},
is byte-identical across turns until the artifact changes.
\end{proof}

\begin{corollary}[Incremental Shared Dependency Expansion]
During top-down generation, when a leaf encounters a new shared dependency $d$ not
in the current shared dependency stream:
(i)~Adding $d$ does not invalidate leaves not depending on $d$;
(ii)~leaves already generated that depend on $d$ are stale
(Proposition~\ref{prop:stale}) and require regeneration;
(iii)~all subsequent leaf generations have $d$ available as shared context.
\end{corollary}

This corollary formalizes the website generation pattern: a new CSS variable
discovered during page generation is added to the design system stream, requires
regenerating only affected pages, and is available to all subsequent pages.

\subsection{Context Stability Hierarchy}

The dual-traversal structure maps directly onto the KV-cache stability hierarchy:
\begin{itemize}
\item \textbf{PERMANENT}: root-level architecture, stable for the goal type's lifetime
\item \textbf{SESSION}: section/module context, byte-identical between semantic changes
\item \textbf{COLD}: multi-level summary tree, changes every $\approx10^2$--$10^3$ messages
\item \textbf{DYNAMIC}: wisdom-program-selected context, varying per turn
\end{itemize}
Root-level architecture $\to$ PERMANENT; section context $\to$ SESSION;
leaf requirements $\to$ DYNAMIC.
Shared context at each hierarchy level is computed once and reused across all
descendant leaf generations.

\section{Observations Schema and Write-Time Extraction}
\label{sec:obs}

An \emph{observations schema} maps stream types $T$ to extraction specifications
$\sigma(T,C)$: \emph{instruction clauses} $[\ell_1,\ldots,\ell_k]$ (natural-language
directives constraining LM extraction output, constituting the stable portion of the
system prompt for type $T$), output schema $\Sigma$, attribute namespace $\pi$,
and staleness condition $\delta$.

Extraction uses a two-part call:
the \emph{stable system prompt} (constant across all extractions of type $T$ and
category $C$, KV-cached by Theorem~\ref{thm:byte}) and the
\emph{dynamic user content} (the specific stream's content and existing attributes,
varying per stream).

\begin{proposition}[Extraction Amortization]
With $k_s$ stable system prompt tokens and $\bar{k}_u$ mean dynamic content tokens,
the mean per-extraction prefill cost at 10\% cache pricing is
$\bar{C}=0.1k_s+\bar{k}_u$.
The asymptotic reduction factor relative to no caching is:
$R=(k_s+\bar{k}_u)/(0.1k_s+\bar{k}_u)\xrightarrow{k_s/\bar{k}_u\to\infty}10.$
\end{proposition}
\begin{proof}
Without caching, the per-extraction prefill cost is $k_s+\bar{k}_u$.
With 10\% cache-hit pricing on the stable prefix, the effective cost is
$0.1k_s + \bar{k}_u$.
The reduction factor is $R = (k_s+\bar{k}_u)/(0.1k_s+\bar{k}_u)$.
Dividing numerator and denominator by $\bar{k}_u$:
$R = (k_s/\bar{k}_u + 1)/(0.1 k_s/\bar{k}_u + 1) \to 1/0.1 = 10$
as $k_s/\bar{k}_u \to \infty$.
\end{proof}

Extracted attributes, once written, serve context assembly, preprocessing, relevance
scoring, autosuggest, and search indexes---all at zero additional LM cost.

\section{Deterministic Semantic Search}
\label{sec:search}

\subsection{The Precision Problem with Embedding Similarity}

For structured domains, typed attribute queries
(``streams with $\alpha(v)[\mathsf{status}]=\mathtt{unconfigured}$'')
have zero false positives by construction---a precision level unachievable by
embedding similarity, which computes undifferentiated similarity over raw content.

\subsection{Write-Time Indexing and Query Expansion}

Keyword extraction is a first-class observation category: the Groker extracts
normalized keyword set $K(v)$ for each $v$, stored as $\alpha(v)[\mathsf{keywords}]$
and indexed into inverted index $\mathrm{Idx}[k]=\{v\mid k\in K(v)\}$.
Query terms are expanded via: stemming; domain ontology traversal (IS-A, SYNONYM-OF
edges as typed graph relations); thesaurus lookup; and cached prior LM expansions.
Search resolves by set intersection: $\mathrm{score}[v]=|Q\cap K(v)|$.

\begin{theorem}[Synonym Cache Convergence]
\label{thm:cache}
Let $\mathcal{W}$ be the finite domain vocabulary with $|\mathcal{W}| < \infty$.
Assume each query contains at least one term ($|\mathcal{T}_q(n)| \geq 1$ for all $n$).
The LM fallback rate $\rho(n)=|\mathcal{T}_q^{\mathrm{new}}(n)|/|\mathcal{T}_q(n)|\to 0$
as $n\to\infty$.
\end{theorem}
\begin{proof}
Let $S_n = \bigcup_{i=1}^n \mathcal{T}_q(i)$ be the set of all vocabulary terms
seen in queries $1$ through $n$.
Since $S_n \subseteq \mathcal{W}$ and $|\mathcal{W}| < \infty$, the sequence
$|S_n|$ is non-decreasing and bounded above by $|\mathcal{W}|$, so it converges.
Therefore $|\mathcal{T}_q^{\mathrm{new}}(n)| = |S_n \setminus S_{n-1}| \to 0$.
Since $|\mathcal{T}_q(n)| \geq 1$ by assumption, the denominator is bounded below,
and $\rho(n) = |\mathcal{T}_q^{\mathrm{new}}(n)|/|\mathcal{T}_q(n)| \leq |\mathcal{T}_q^{\mathrm{new}}(n)| \to 0$.
\end{proof}

Theorem~\ref{thm:cache} mirrors Theorem~\ref{thm:mono}: in both cases, LM cost per
operation converges to zero for finite-domain workloads as the relevant cache
accumulates coverage.

\section{Implementation}
\label{sec:impl}

The architecture is realized across three plugins.
\textbf{Qbix} provides the stream graph substrate: \texttt{Streams::relate()},
\texttt{Streams\_Category} (maintained transactionally), \texttt{Users\_Vote}.
\textbf{Safebox} provides the sandbox (realizing SPACER's $C$/Capabilities and
$E$/Execution components~\cite{magarshak2026mm}): \texttt{Protocol.LLM} with
Anthropic and OpenAI adapters, \texttt{Action.propose} governed write pipeline,
\texttt{Protocol.Transcription}, \texttt{Protocol.Image}.
\textbf{Safebots} provides the Groker enrichment layer, goal stream architecture,
multi-level summary tree, wisdom library management, and preprocessing pipeline.

\textit{KV-cache support}: For Anthropic models, PERMANENT and SESSION blocks
are passed as system content blocks with \texttt{cache\_control:\{type:"ephemeral"\}},
establishing explicit breakpoints.
For OpenAI models, they are merged into a single system message to maximize stable
prefix length for automatic prefix caching.

\section{Discussion}
\label{sec:discussion}

\textbf{Correctness assumptions.}
Theorem~\ref{thm:comp} assumes individually correct Groker invocations.
In practice, Grokers use LMs and are stochastic; the governed write pipeline mitigates
this via schema validation and fitness-based evolution.
``Correct'' means operationally: passes schema validation and maintains fitness
$\geq f_{\min}$ over the evaluation window.

\textbf{Finite vocabulary assumption.}
Theorem~\ref{thm:cache} requires $|\mathcal{W}|<\infty$.
For open-domain deployments the vocabulary is unbounded;
in this case $\rho(n)$ converges to the fraction of novel user terms that have no
ontology expansion, which may remain bounded away from zero indefinitely.
The theorem applies to structured knowledge domains (product catalogs, codebases,
organizational ontologies) where the domain vocabulary is finite and indexable.

\textbf{Acyclic dependency assumption.}
Theorem~\ref{thm:dual} assumes acyclic dependency graphs.
Cyclic dependencies (mutual imports, circular references) require cycle-breaking
heuristics; we leave this extension to future work.

\textbf{Architecture stack.}
The full Magarshak Architecture is:
$\underbrace{\MM}_{\text{substrate}}\to\underbrace{\text{Grokers}}_{\text{comprehension}}\to\underbrace{\text{Context}}_{\text{intelligence}}$~\cite{magarshak2026context}.
This paper covers the comprehension layer; the intelligence layer---proactive
goal-directed agents, organizational efficiency theorems, cross-platform
governance---is treated in~\cite{magarshak2026context}.

\section{Conclusion}
\label{sec:conclusion}

\textsc{Grokers} provides a write-time intelligence architecture with three proved
properties.
The Byte-Identity Theorem establishes near-100\% KV-cache hit rates on deterministically
assembled stable prefixes---fundamentally unachievable by retrieval-based systems.
The Accumulation Monotonicity Theorem establishes that the wisdom library covers an
increasing fraction of interactions without LM calls, with marginal LM cost
non-increasing over time---a property no existing chatbot architecture exhibits.
The Dual-Traversal Ordering Theorem establishes the duality of generation and
comprehension as the unique correct orderings for their respective tasks, composing
into a complete cycle in which comprehension output becomes cached context for
future generation.
These properties constitute a formal basis for write-time intelligence as a strictly
more efficient alternative to query-time retrieval for structured knowledge domains
with recurring interaction patterns.

\section*{Acknowledgements}
The stream graph substrate is implemented in the Qbix open-source framework,
which realizes the Magarshak Machine SPACER substrate~\cite{magarshak2026mm}.

\bibliographystyle{IEEEtran}
\bibliography{references}

\end{document}